\documentclass[10pt]{article}

\usepackage[letterpaper,textwidth=6in,textheight=8in]{geometry}
\usepackage{amsmath,amssymb,amsthm,mathtools,bm}
\usepackage{booktabs}
\usepackage{graphicx}
\usepackage{float}
\usepackage{enumitem}
\usepackage{microtype}
\usepackage{xcolor}
\usepackage{hyperref}

\hypersetup{
  colorlinks=true,
  linkcolor=blue,
  citecolor=blue,
  urlcolor=blue
}

\newtheorem{theorem}{Theorem}[section]
\newtheorem{proposition}[theorem]{Proposition}
\newtheorem{lemma}[theorem]{Lemma}
\newtheorem{corollary}[theorem]{Corollary}
\theoremstyle{definition}
\newtheorem{definition}[theorem]{Definition}

\newcommand{\E}{\mathbb{E}}
\newcommand{\Pp}{\mathbb{P}}
\newcommand{\N}{\mathbb{N}}
\newcommand{\C}{\mathbb{C}}

\newcommand{\A}{\mathcal{A}}
\newcommand{\Sset}{\mathcal{S}}
\newcommand{\Tset}{\mathcal{T}}
\newcommand{\K}{\mathcal{K}}
\newcommand{\G}{\mathcal{G}}
\newcommand{\Hh}{\mathcal{H}}
\newcommand{\Dset}{\mathcal{D}}
\newcommand{\diag}{\operatorname{diag}}
\newcommand{\rank}{\operatorname{rank}}

\newcommand{\Var}{\operatorname{Var}}
\newcommand{\ii}{\mathrm{i}}
\newcommand{\balpha}{\bm{\alpha}}
\newcommand{\bbeta}{\bm{\beta}}
\newcommand{\bt}{\bm{t}}
\newcommand{\br}{\bm{r}}
\newcommand{\btheta}{\bm{\theta}}
\newcommand{\bxi}{\bm{\xi}}
\newcommand{\bXi}{\bm{\Xi}}
\newcommand{\beps}{\bm{\varepsilon}}
\newcommand{\bmone}{\bm{1}}
\newcommand{\norm}[1]{\left\lVert #1\right\rVert}

\title{\bf Coded Hankel Polynomial Chaos:\\
Spectral Identification of Dominant Polynomial-Chaos Modes}

\author{Zhiliang Deng\thanks{School of Mathematical Science, University of Electronic Science and Technology of China, Email: dengzhl@uestc.edu.cn} \and Xiaomei Yang\thanks{School of Mathematics, Southwest Jiaotong University, Email: yangxiaomath@swjtu.edu.cn}}
\date{}

\begin{document}
\maketitle

\begin{abstract}
Identification of dominant polynomial-chaos modes is usually formulated as a
sparse-regression problem on a sampled multivariate polynomial dictionary.  We
develop coded Hankel polynomial chaos (CH-PC), a complementary spectral
formulation for dominant-mode identification.  A finite generating transform
converts PCE coefficients into a coefficient-generating polynomial, and
evaluation along a geometric phase orbit produces a finite exponential sum.
Its model order and spectral nodes are encoded by low-rank Hankel matrices,
while coordinate phase shifts attach root-of-unity labels from which the full
polynomial multi-indices are recovered.  Coordinate-shifted probes are combined
as common-node snapshots, and independent phase encodings provide redundant
representations when a single spectral encoding is poorly conditioned.  For
finite observations, population, finite-data, and observed probes are kept
distinct: sampling or quadrature error and observation error enter as separate
Hankel perturbations, which are then connected to spectral stability, discrete
decoding, and phase voting.  For tensor-product candidate sets, the generating
kernel factorizes into one-dimensional sums and can be evaluated without
assembling the full multivariate PCE design matrix.  Numerical experiments on sparse Legendre benchmarks and a stochastic Darcy
problem illustrate exact recovery, noise stabilization, unknown-order
identification by phase persistence, and dominant-mode recovery for a
PDE-generated quantity of interest.
\end{abstract}



\section{Introduction}
\label{sec:intro}

Polynomial chaos expansion (PCE) represents stochastic model responses in an orthogonal polynomial basis determined by the probability law of the uncertain inputs \cite{Wiener1938, XiuKarniadakis2002}.  In high-dimensional settings, however, even a modest  polynomial degree can produce  a large number of candidate multivariate basis modes, while the response of interest often depends significantly on only a small subset of them. 
Identifying this influential subset is therefore important:  it not only yields a parsimonious representation but also reveals the key variables, polynomial orders, and interactions that drive the response. To address this basis-selection problem, a range of sparse-regression techniques have been developed, including least angle regression (LARS), orthogonal matching pursuit (OMP), $\ell_1$ regularization, compressed sensing, and adaptive greedy strategies \cite{BlatmanSudret2011, DoostanOwhadi2011, EfronEtAl2004, LuthenMarelliSudret2021, RauhutWard2012, TroppGilbert2007}.

This paper develops a complementary spectral viewpoint on dominant-mode
identification based on Hankel matrix-pencil techniques. The underlying spectral principle can be traced back to classical Prony methods,  which recover finite exponential sums from
structured measurements. In this setting, Hankel matrices provide a low-rank representation of the exponential model, while matrix-pencil techniques recover the associated spectral nodes.
 Generalized Prony methods extend this  basic principle beyond exponential sums to
sparse expansions in more general function systems
\cite{PeterPlonka2013,StampferPlonka2020}.  Related spectral recovery constructions have
also been developed for sparse Legendre and Gegenbauer expansions and for
multivariate Prony systems \cite{KunisEtAl2016, PeterPlonkaRosca2013, PottsTasche2016}.  Thus, Prony-type recovery, Hankel low-rank structure, and matrix-pencil techniques are established ingredients rather than new contributions of the present work. 
 The problem addressed here is different: how can ordinary PCE input--output data be transformed into a structured spectral representation whose low-rank Hankel structure reveals the number of dominant modes and whose spectral information retains their polynomial multi-indices, so that support identification can be performed by spectral recovery rather than sparse regression?

The construction begins with a finite generating transform that converts the
PCE coefficients into a multivariate coefficient-generating polynomial.
Evaluation along a geometric phase orbit maps the active multi-indices to a
finite exponential sum, whose model order is identified from Hankel rank and
whose spectral nodes are recovered by a matrix pencil.  A single spectral node,
however, contains a multi-index only through a scalar phase combination.  We
therefore introduce coordinate phase shifts that preserve the nodes while
multiplying their amplitudes by root-of-unity factors.  These factors encode the
individual polynomial degrees and reduce recovery of a multivariate index to
finite coordinate-wise decoding.  The shifted probes also provide multiple
snapshots with the same nodes.  CH-PC combines these common-node snapshots and,
when a single phase is poorly conditioned, repeats the same construction with
independent phase encodings.

The exact spectral construction has a natural finite-data realization.
Data-based probes approximate the required generating functionals, and the
resulting perturbations can be tracked at the joint-Hankel level before being
propagated through spectral recovery and discrete multi-index decoding.  This
separation is useful because the conditioning of a phase is determined not only
by the data error but also by the separation of its encoded nodes.  Repeating
the decoder over independent phases therefore supplies a second form of
redundancy, and support persistence across phases provides a practical route to
stable identification and unknown model order.  CH-PC also has a computational
structure different from dictionary-based sparse regression.  For tensor-product
candidate sets, the generating kernel factorizes into one-dimensional sums, so
spectral probes can be formed without assembling the full multivariate
polynomial dictionary.  This yields a different memory and scaling profile,
although it does not imply a universal runtime advantage.

The main contributions are fourfold.  First, a PCE-specific generating
representation converts dominant polynomial-chaos modes into a finite
exponential system and yields an exact Hankel characterization of the effective
model order.  Second, coordinate phase shifts provide a root-of-unity encoding
that reconstructs the individual degrees and hence the full active
multi-indices.  Third, common-node snapshots and independent phase encodings
provide a redundant spectral recovery mechanism with explicit separation and
support-recovery results.  Fourth, the finite-data analysis separates
integration error from propagated observation error and connects the resulting
Hankel perturbations to model-order identification, discrete decoding, and phase
voting.  The framework is then extended to approximately sparse expansions and
to sensitivity information.  Section~\ref{sec:formulation} develops the exact
spectral representation, Section~\ref{sec:robust} treats the finite-data coded
Hankel method and its stability, Section~\ref{sec:approx-sparse} discusses
approximate sparsity, sensitivity, and methodological scope, and
Section~\ref{sec:numerics} presents the numerical experiments.

\section{Dominant polynomial-chaos modes and exact spectral representation}
\label{sec:formulation}

This section formalizes the object to be identified and derives the exact
population-level representation on which CH-PC is based.  We define dominant
PCE modes through orthogonal coefficient energy, construct the finite
generating-kernel transform, and show how phase coding converts the active
multi-indices into a finite Prony system with an exact Hankel factorization.
Finite observations and their perturbations are introduced only in
Section~\ref{sec:robust}.

\subsection{Polynomial-chaos model and dominant modes}

Let $\bXi=(\Xi_1,\ldots,\Xi_d)$ be a random input vector with probability
measure $\rho$.  We first assume independent components, so that
$\rho=\rho_1\otimes\cdots\otimes\rho_d$.  For the $\ell$th coordinate, let
$\{\psi_n^{(\ell)}\}_{n\ge0}$ be orthonormal in $L^2(\rho_\ell)$.  For a
multi-index $\balpha=(\alpha_1,\ldots,\alpha_d)\in\N_0^d$, define
$\Psi_{\balpha}(\bxi)=\prod_{\ell=1}^{d}
\psi_{\alpha_\ell}^{(\ell)}(\xi_\ell)$.  The tensor-product system then
satisfies
$\E[\Psi_{\balpha}(\bXi)\Psi_{\bbeta}(\bXi)]=\delta_{\balpha\bbeta}$.
Let $\A\subset\N_0^d$ be a finite candidate set and write
$\mathcal P_{\A}=\operatorname{span}\{\Psi_{\balpha}:\balpha\in\A\}$.
For a scalar quantity of interest $Y=Y(\bXi)\in L^2(\rho)$, its orthogonal
projection onto $\mathcal P_{\A}$ is
\begin{equation}
Y_{\A}(\bXi)=\sum_{\balpha\in\A}c_{\balpha}\Psi_{\balpha}(\bXi),
\qquad
c_{\balpha}=\E\left[Y(\bXi)\Psi_{\balpha}(\bXi)\right].
\label{eq:pce-projection}
\end{equation}

Throughout, $\bXi$ denotes the random input governed by $\rho$, while $\bxi$
denotes a realization.  This distinction will become important when the
population generating functionals are replaced by finite-data approximations
in Section~\ref{sec:robust}.

Because the basis is orthonormal, coefficient magnitude has a direct energetic
interpretation.

\begin{proposition}[Orthogonal tail identity]
\label{prop:tail-identity}
For any subset $\Sset\subset\A$, define
$Y_{\Sset}=\sum_{\balpha\in\Sset}c_{\balpha}\Psi_{\balpha}$.  Then
\begin{equation*}
\norm{Y_{\A}-Y_{\Sset}}_{L^2(\rho)}^2=\sum_{\balpha\in\A\setminus\Sset}|c_{\balpha}|^2.
\end{equation*}
\end{proposition}

\begin{proof}
The result is Parseval's identity on the finite orthonormal system.
\end{proof}

We call $\Sset$ a \emph{dominant support} when it is small relative to $\A$
and captures the principal coefficient energy.  In the exactly sparse case,
$c_{\balpha}=0$ for $\balpha\notin\Sset$.  In the approximately sparse case the
omitted coefficients need not vanish, but we require
\begin{equation*}
\sum_{\balpha\in\A\setminus\Sset}|c_{\balpha}|^2\ll\sum_{\balpha\in\Sset}|c_{\balpha}|^2.
\end{equation*}
The primary identification problem is to determine a dominant support
$\Sset\subset\A$, its effective cardinality $s=|\Sset|$, and the associated
coefficients from input--output information.  The multi-index itself carries
structural information: for example, $\balpha=(3,0,0)$ represents a third-order
contribution in the first input, whereas $\balpha=(1,1,0)$ represents an
interaction between the first two inputs.

\subsection{Finite generating-kernel transform}

For $\bt=(t_1,\ldots,t_d)\in\C^d$, write
$\bt^{\balpha}=\prod_{\ell=1}^{d}t_\ell^{\alpha_\ell}$.

\begin{definition}[Finite PCE generating kernel]
For a finite candidate set $\A$, define
\begin{equation*}
\K_{\A}(\bt,\bxi)
=
\sum_{\balpha\in\A}
\bt^{\balpha}
\Psi_{\balpha}(\bxi),
\end{equation*}
and the generating transform
\begin{equation*}
\G_{\A}Y(\bt)
=
\E
\left[
Y(\bXi)
\K_{\A}(\bt,\bXi)
\right].
\end{equation*}
\end{definition}

\begin{theorem}[Coefficient-generating identity]
\label{thm:generating-identity}
For every $Y\in L^2(\rho)$ and finite $\A$,
\begin{equation*}
\G_{\A}Y(\bt)
=
\sum_{\balpha\in\A}
c_{\balpha}\bt^{\balpha},
\end{equation*}
where $c_{\balpha}$ are the coefficients in \eqref{eq:pce-projection}.
\end{theorem}

\begin{proof}
Because $\A$ is finite,
\begin{align*}
\G_{\A}Y(\bt)=\sum_{\balpha\in\A}
\bt^{\balpha}
\E\left[Y(\bXi)\Psi_{\balpha}(\bXi)\right]=\sum_{\balpha\in\A}c_{\balpha}\bt^{\balpha}.
\end{align*}
\end{proof}

Thus the transform converts the orthogonal PCE coefficients into an ordinary
multivariate coefficient-generating polynomial while preserving the complete
multi-index support.  No derivative of $Y$ is required, and the construction
is independent of the particular orthogonal-polynomial family.

\subsection{Tensorized evaluation of the generating probes}
\label{subsec:fast-probes}

The generating transform is useful computationally only if its probe values can
be formed without replacing one large PCE regression problem by another large
sum.  For tensor-product candidate sets, the kernel has an exact separable
structure that removes the explicit dependence on the full multivariate
cardinality at the level of a single kernel evaluation.

\begin{proposition}[Tensor-product factorization of the PCE kernel]
\label{prop:tensor-kernel}
Let
$\A=\prod_{\ell=1}^{d}\{0,\ldots,p_\ell\}$ and define the one-dimensional
finite kernels
$$
K_{p_\ell}^{(\ell)}(t_\ell,\xi_\ell)
=
\sum_{r=0}^{p_\ell}
t_\ell^r\psi_r^{(\ell)}(\xi_\ell).
$$
Then
$$
\K_{\A}(\bt,\bxi)
=
\prod_{\ell=1}^{d}
K_{p_\ell}^{(\ell)}(t_\ell,\xi_\ell).
$$
\end{proposition}

\begin{proof}
Expanding the product of the $d$ one-dimensional sums gives one term
$\bt^{\balpha}\Psi_{\balpha}(\bxi)$ for every
$\balpha\in\prod_{\ell=1}^{d}\{0,\ldots,p_\ell\}$ and no other terms.
\end{proof}

Let $P=|\A|=\prod_{\ell=1}^{d}(p_\ell+1)$.  If the univariate basis tables
$\psi_r^{(\ell)}(\xi_\ell^{(n)})$ are precomputed on $N$ observations, one
probe point $\bt$ can be evaluated by forming $d$ one-dimensional sums and
multiplying them.  The arithmetic scales as
$O(N\sum_{\ell=1}^{d}(p_\ell+1))$, rather than $O(NP)$ for an explicit
summation over all multivariate basis functions.  The corresponding stored
basis tables require $O(N\sum_{\ell=1}^{d}(p_\ell+1))$ entries instead of the
$O(NP)$ entries of a full multivariate design matrix.  Orthogonal-polynomial
three-term recurrences may further reduce basis-evaluation overhead.

For $N_{\mathrm{pr}}$ phase-orbit and coordinate-shift probe points, a direct implementation
therefore has cost
$O(NN_{\mathrm{pr}}\sum_{\ell=1}^{d}(p_\ell+1))$ after the univariate tables have been
formed.  The subsequent Hankel SVDs and reduced pencils involve matrices whose
working rank is controlled by the model-order ceiling $s_{\max}$, not by $P$.
This is the main source of a possible scaling advantage over a method that
repeatedly searches an $N\times P$ dictionary.  The statement is deliberately
restricted to tensor-product candidate sets.  For total-degree or general
finite sets the exact product factorization need not hold, and one may use a
direct finite sum, dynamic programming, or another structured evaluation
scheme.  Consequently, runtime comparisons are meaningful only when the
candidate geometry and the number of CH-PC probes are reported together.

\subsection{Phase-coded Prony representation and coordinate decoding}

Throughout this subsection and the next, ``exact'' refers to the population
probe values generated by $\G_{\A}Y$.  A finite set of noise-free observations
generally produces a quadrature or sampling approximation to these probes;
that distinction is treated in Section~\ref{subsec:noise-propagation}.

Choose radii $0<r_\ell\le1$, $\ell=1,\ldots,d$, and a phase vector
$\btheta=(\theta_1,\ldots,\theta_d)\in[0,2\pi)^d$.  For $k\ge0$, define the
geometric phase orbit by $t_\ell(k)=r_\ell\exp(\ii k\theta_\ell)$.  Evaluating
the generating transform along this orbit gives
\begin{equation*}
m_k=\G_{\A}Y\left(\bt(k)\right)=\sum_{\balpha\in\A}c_{\balpha}\br^{\balpha}\exp\left(\ii k\btheta\cdot\balpha\right).
\end{equation*}
For an exactly $s$-sparse support
$\Sset=\{\balpha^{(1)},\ldots,\balpha^{(s)}\}$, write
$c^{(j)}:=c_{\balpha^{(j)}}$ and set
$a^{(j)}=c^{(j)}\br^{\balpha^{(j)}}$ and
$z^{(j)}=\exp(\ii\btheta\cdot\balpha^{(j)})$, $j=1,\ldots,s$.
The probe sequence becomes
\begin{equation}
m_k=\sum_{j=1}^{s}a^{(j)}\left(z^{(j)}\right)^k,
\label{eq:prony-sequence}
\end{equation}
which is a finite exponential sum.  The number of terms $s$ is not assumed
known.  In exact arithmetic one may select a model-order ceiling $s_{\max}$,
choose Hankel dimensions $R,C\ge s_{\max}$, and recover the actual $s$ from
the Hankel rank.  The lag-zero and lag-one Hankel pair requires the probe
values $m_0,\ldots,m_{R+C-1}$.  For a fixed probe budget, balanced choices
$R\approx C$ maximize the largest identifiable model order.

The scalar node $z^{(j)}$ contains the active multi-index only through the
one-dimensional phase combination $\btheta\cdot\balpha^{(j)}$.  Although this
scalar code is generically injective on a finite candidate set, direct matching
can be poorly conditioned when two phase codes are close.  We therefore
introduce coordinate shifts that leave every Prony node $z^{(j)}$ unchanged
while multiplying its amplitude by a finite root-of-unity code.  This converts
multi-index recovery into $d$ coordinate-wise discrete decoding problems and,
at the same time, creates multiple snapshots sharing the same spectral nodes.

To recover the individual coordinates of the active multi-indices, define,
for each $\ell=1, \ldots, d$,
\begin{equation*}
p_\ell=\max_{\balpha\in\A}\alpha_\ell,\qquad
\gamma_\ell=\frac{2\pi}{p_\ell+1}.
\end{equation*}
The base probe orbit is regarded as snapshot $\nu=0$, namely
\begin{equation*}
\bt^{[0]}(k)=\bt(k).
\end{equation*}
For each coordinate $\ell=1, \ldots, d$, define the $\ell$-th shifted probe
point $\bt^{[\ell]}(k)$ by multiplying only the $\ell$-th component of
$\bt(k)$ by $\exp(\ii\gamma_\ell)$:
\begin{equation*}
t_r^{[\ell]}(k)=
\begin{cases}
\exp(\ii\gamma_\ell)t_\ell(k), & r=\ell,\\
t_r(k), & r\neq\ell,
\end{cases}
\qquad r=1, \ldots, d.
\end{equation*}
Evaluating the generating transform at this shifted probe point gives
\begin{equation}
m_k^{[\ell]}=\G_{\A}Y\left(\bt^{[\ell]}(k)\right)=\sum_{j=1}^{s}a^{(j)}\chi_\ell^{(j)}\left(z^{(j)}\right)^k,
\qquad
\ell=1, \ldots, d,
\label{eq:coordinate-sequence}
\end{equation}
where
\begin{equation*}
\chi_\ell^{(j)}=\exp\left(\ii\gamma_\ell\alpha_\ell^{(j)}\right),
\qquad
j=1, \ldots, s.
\end{equation*}
A coordinate with $p_\ell=0$ has the known degree
$\alpha_\ell=0$ for every candidate and requires no decoding.  We therefore
omit such trivial coordinates from the decoding analysis and assume
$p_\ell\ge1$ for the retained coordinates.  For each retained coordinate
$\ell$, the choice $\gamma_\ell=2\pi/(p_\ell+1)$ maps the admissible degrees
$0, \ldots, p_\ell$ bijectively onto the $(p_\ell+1)$-st roots of unity.
Consequently, $\chi_\ell^{(j)}$ uniquely encodes the coordinate
$\alpha_\ell^{(j)}$.  Moreover, the corresponding coordinate codebook has
minimum pairwise distance
\begin{equation*}
d_\ell=2\sin\left(\frac{\pi}{p_\ell+1}\right).
\end{equation*}
Thus the $d$ shifted snapshots convert recovery of the multi-index
$\balpha^{(j)}$ into $d$ finite, coordinate-wise root-of-unity decoding
problems while preserving the common Prony nodes $z^{(j)}$.

\subsection{Exact Hankel factorization and recovery}

For convenience,  set $m_k^{[0]}=m_k$, while $m_k^{[\ell]}$, $\ell=1, \ldots, d$,  denotes the  coordinate-shifted
sequence defined in \eqref{eq:coordinate-sequence}.

For any positive integer $K$, define the $K\times s$ Vandermonde matrix
$\mathcal V_K=\bigl((z^{(j)})^i\bigr)_{0\le i\le K-1,\,1\le j\le s}$.
Thus $\mathcal V_R$ and $\mathcal V_C$ differ only in their numbers of rows,
which match the row and column dimensions of the Hankel matrices.  For
$R,C\ge s$, define
\begin{equation*}
\mathcal H_\kappa^{[\nu]}=
\left(m_{u+v+\kappa}^{[\nu]}\right)_{0\le u\le R-1,\, 0\le v\le C-1},
\qquad
\kappa=0, 1,
\quad
\nu=0, \ldots, d.
\end{equation*}
In particular, $\mathcal H_0^{[0]}$ is the base Hankel matrix,
$\mathcal H_1^{[0]}$ is its one-step lagged companion, and
$\mathcal H_0^{[\ell]}$ is the unlagged Hankel matrix generated by the
$\ell$-th coordinate-shifted probe sequence.

Set $D=\diag(a^{(1)},\ldots,a^{(s)})$ and
$Z=\diag(z^{(1)},\ldots,z^{(s)})$.  To represent the coordinate snapshots,
let $\Omega_0=I_s$ and
$\Omega_\ell=\diag(\chi_\ell^{(1)},\ldots,\chi_\ell^{(s)})$ for
$\ell=1,\ldots,d$.  We also set $\chi_0^{(j)}=1$ and
$a_\nu^{(j)}=a^{(j)}\chi_\nu^{(j)}$ for $\nu=0,\ldots,d$; hence
$a_0^{(j)}=a^{(j)}$ is the base amplitude and
$a_\ell^{(j)}=a^{(j)}\chi_\ell^{(j)}$ is the amplitude of the $j$th active
mode in the $\ell$th coordinate snapshot.  Then
\begin{equation}
\mathcal H_\kappa^{[\nu]}=
\mathcal V_R D\Omega_\nu Z^\kappa
\mathcal V_C^T,
\qquad
\kappa=0, 1,
\quad
\nu=0, \ldots, d.
\label{eq:exact-hankel-factorization}
\end{equation}
Because $\Omega_\nu$ and $Z$ are diagonal, their order is immaterial; the
displayed order separates the coordinate-code factor from the lag factor.

\begin{theorem}[Exact model-order and single-phase recovery]
\label{thm:exact-single}
Assume that
$a^{(j)}\neq0$, $j=1,\ldots,s$, and that the nodes
$z^{(1)},\ldots,z^{(s)}$ are pairwise distinct.
Let $s\le s_{\max}$ and choose $R,C\ge s_{\max}$.
Then the model order, spectral nodes, active multi-indices, and PCE
coefficients are uniquely determined as follows.

First,
\begin{equation}
\rank\left(\mathcal H_0^{[0]}\right)=s,
\label{eq:exact-rank}
\end{equation}
so the unknown model order is determined by the Hankel rank.

Let $U_s\in\C^{R\times s}$ and $V_s\in\C^{C\times s}$ be orthonormal
bases for the left and right signal subspaces of
$\mathcal H_0^{[0]}$.  Then
$U_s^*\mathcal H_0^{[0]}V_s$ is nonsingular, and the generalized
eigenvalues of the reduced pencil
\begin{equation}
\left(
U_s^*\mathcal H_1^{[0]}V_s,\,
U_s^*\mathcal H_0^{[0]}V_s
\right)
\label{eq:compressed-pencil}
\end{equation}
are exactly
\begin{equation}
\left\{
z^{(1)},\ldots,z^{(s)}
\right\}.
\label{eq:exact-node-recovery}
\end{equation}

Once the nodes are known, the amplitudes of each snapshot
$\nu=0,\ldots,d$ are uniquely determined from
\begin{equation}
\begin{pmatrix}
m_0^{[\nu]}
\\
m_1^{[\nu]}
\\
\vdots
\\
m_{s-1}^{[\nu]}
\end{pmatrix}
=
\mathcal V_s
\begin{pmatrix}
a_\nu^{(1)}
\\
a_\nu^{(2)}
\\
\vdots
\\
a_\nu^{(s)}
\end{pmatrix}.
\label{eq:exact-amplitude-recovery}
\end{equation}
The coordinate codes and active polynomial degrees are then recovered by
\begin{equation}
\chi_\ell^{(j)}=\frac{a_\ell^{(j)}}{a_0^{(j)}},
\qquad
\alpha_\ell^{(j)}=\operatorname*{argmin}_{r\in\{0,\ldots,p_\ell\}}\left|
\chi_\ell^{(j)}-\exp\left(\ii\gamma_\ell r\right)\right|,
\label{eq:exact-coordinate-decoding}
\end{equation}
for $\ell=1,\ldots,d$ and $j=1,\ldots,s$.
Finally,
\begin{equation}
c^{(j)}=\frac{a_0^{(j)}}{\br^{\balpha^{(j)}}},
\qquad
j=1, \ldots, s.
\label{eq:exact-coefficient-recovery}
\end{equation}
Hence the support
$\{\balpha^{(1)},\ldots,\balpha^{(s)}\}$ and the corresponding PCE
coefficients are uniquely recovered.
\end{theorem}

\begin{proof}
By \eqref{eq:exact-hankel-factorization},
\begin{equation*}
\mathcal H_0^{[0]}
=
\mathcal V_R D\mathcal V_C^T.
\end{equation*}
Since the nodes are pairwise distinct,
$\mathcal V_R$ and $\mathcal V_C$ have full column rank, while
$a^{(j)}\neq0$ implies that $D$ is nonsingular.  Hence
\begin{equation*}
\rank\left(\mathcal H_0^{[0]}\right)=s,
\end{equation*}
which proves \eqref{eq:exact-rank}.

We next recover the spectral nodes.  Since $U_s$ and $V_s$ span the left and
right signal subspaces of $\mathcal H_0^{[0]}$, respectively, the matrices
\begin{equation*}
M_L
=
U_s^*\mathcal V_R,
\qquad
M_R
=
\mathcal V_C^T V_s
\end{equation*}
are nonsingular.  Using \eqref{eq:exact-hankel-factorization} for
$\kappa=0$ and $\kappa=1$ gives
\begin{align*}
U_s^*\mathcal H_0^{[0]}V_s=M_LDM_R,\qquad
U_s^*\mathcal H_1^{[0]}V_s=M_LDZM_R.
\end{align*}
Therefore
\begin{align*}
\det\left(U_s^*\mathcal H_1^{[0]}V_s-\lambda U_s^*\mathcal H_0^{[0]}V_s\right)=\det(M_LD)\det(Z-\lambda I_s)\det(M_R).
\end{align*}
Hence the generalized eigenvalues of the reduced pencil are exactly
$z^{(1)}, \ldots, z^{(s)}$.

Since the nodes are pairwise distinct, $\mathcal V_s$ is nonsingular, so
\eqref{eq:exact-amplitude-recovery} uniquely determines the amplitude vector
for each snapshot.  The root-of-unity encoding in
\eqref{eq:coordinate-sequence} then yields
\eqref{eq:exact-coordinate-decoding}.  Finally,
\eqref{eq:exact-coefficient-recovery} follows from
\begin{equation*}
a_0^{(j)}
=
c^{(j)}\br^{\balpha^{(j)}}.
\end{equation*}
\end{proof}


\begin{proposition}[Almost-sure injectivity of phase encoding]
\label{prop:generic-phase}
Let $\A\subset\N_0^d$ be finite and let $\btheta$ have an absolutely
continuous distribution on $[0,2\pi)^d$.  Then, with probability one,
the map
\begin{equation*}
\balpha
\longmapsto
\exp\left(\ii\btheta\cdot\balpha\right)
\end{equation*}
is injective on $\A$.
\end{proposition}

\begin{proof}
For any distinct $\balpha,\bbeta\in\A$, a collision requires
\begin{equation*}
\btheta\cdot(\balpha-\bbeta)
\in
2\pi\mathbb Z.
\end{equation*}
This event is contained in a finite union of measure-zero hyperplanes.
Since $\A$ is finite, a union over all distinct pairs still has measure zero.
\end{proof}

Thus the distinct-node assumption in
Theorem~\ref{thm:exact-single} holds almost surely under a random
absolutely continuous phase choice.

\section{Noise-robust coded Hankel recovery}
\label{sec:robust}

This section passes from the exact population representation to the practical
finite-data decoder.  The logic is sequential.  We first use all coordinate
snapshots to form a common-node joint Hankel pair.  We then reduce finite-data
and observation errors to perturbations of this pair.  Random phases quantify
and diversify the phase-dependent spectral conditioning, and the resulting
perturbation and separation scales are combined with discrete coordinate
decoding and phase voting.  The section ends with coefficient refitting and a
complete CH-PC algorithm.

\subsection{Joint-snapshot Hankel representation}

The coordinate-shifted probe sequences introduced for multi-index decoding
share exactly the same Prony nodes as the base sequence.  In the exact
recovery of Section~\ref{sec:formulation}, the nodes can already be identified
from the base Hankel pair alone.  For finite-data recovery, however, all
$d+1$ probe sequences are available and can be used jointly to estimate the
common spectral subspace.  We therefore stack their Hankel matrices before
performing the spectral recovery.
 Define
\begin{equation*}
\mathcal J_\kappa
=
\begin{pmatrix}
\mathcal H_\kappa^{[0]}\\
\mathcal H_\kappa^{[1]}\\
\vdots\\
\mathcal H_\kappa^{[d]}
\end{pmatrix},
\qquad
\kappa=0,1,
\end{equation*}
and
\begin{equation*}
\mathcal W
=
\begin{pmatrix}
\mathcal V_RD\Omega_0\\
\mathcal V_RD\Omega_1\\
\vdots\\
\mathcal V_RD\Omega_d
\end{pmatrix}.
\end{equation*}
The exact factorization \eqref{eq:exact-hankel-factorization} gives
\begin{equation*}
\mathcal J_0
=
\mathcal W\mathcal V_C^T,
\qquad
\mathcal J_1
=
\mathcal WZ\mathcal V_C^T.
\end{equation*}

\begin{proposition}[Common-node joint Hankel representation]
\label{prop:joint-snapshot}
Under the assumptions of Theorem~\ref{thm:exact-single},
$\rank(\mathcal J_0)=s$.  Let
$U_{J,s}\in\C^{(d+1)R\times s}$ and $V_{J,s}\in\C^{C\times s}$ be
orthonormal bases for the left and right signal subspaces of
$\mathcal J_0$.  Then $U_{J,s}^*\mathcal J_0V_{J,s}$ is nonsingular, and the
generalized eigenvalues of the reduced joint pencil
\begin{equation}
\left(
U_{J,s}^*\mathcal J_1V_{J,s},\,
U_{J,s}^*\mathcal J_0V_{J,s}
\right)
\label{eq:joint-compressed-pencil}
\end{equation}
are exactly
$\{z^{(1)},\ldots,z^{(s)}\}$.
\end{proposition}

\begin{proof}
The first block of $\mathcal W$ is $\mathcal V_RD$, which has full column
rank, so $\mathcal W$ and $\mathcal V_C$ both have rank $s$.  Hence
$\rank(\mathcal J_0)=s$.  Since $U_{J,s}$ and $V_{J,s}$ span the left and
right signal subspaces, respectively,
\begin{equation*}
M_{J, L}=U_{J, s}^*\mathcal W,
\qquad
M_{J, R}=\mathcal V_C^TV_{J, s}
\end{equation*}
are nonsingular.  Therefore
\begin{align*}
U_{J, s}^*\mathcal J_0V_{J, s}=M_{J, L}M_{J, R},\qquad
U_{J, s}^*\mathcal J_1V_{J, s}=M_{J, L}ZM_{J, R}.
\end{align*}
It follows that
\begin{align*}
\det\left(U_{J, s}^*\mathcal J_1V_{J, s}-\lambda U_{J, s}^*\mathcal J_0V_{J, s}\right)=\det(M_{J, L})\det(Z-\lambda I_s)\det(M_{J, R}),
\end{align*}
which proves the claim.
\end{proof}

In finite precision, an SVD of the perturbed $\mathcal J_0$ estimates the
common signal subspace, and the reduced pair in
\eqref{eq:joint-compressed-pencil} is formed from the perturbed joint Hankel
matrices.  Once the common nodes are recovered, Vandermonde least-squares fits
are performed separately for the base and coordinate-shifted probe sequences.
The ratios of shifted to base amplitudes then estimate
$\chi_\ell^{(j)}$.  Thus joint stacking changes the spectral estimation step,
not the subsequent coordinate-decoding rule.

\subsection{Finite-data probes, observation noise, and Hankel perturbation}
\label{subsec:noise-propagation}

The finite-data stage starts from observations
$\Dset_N=\{(\bxi^{(n)},y^{(n)})\}_{n=1}^{N}$, where $\bxi^{(n)}$ is a
realization of $\bXi$.  Unless stated otherwise, we use
\begin{equation}
y^{(n)}
=
Y\left(\bxi^{(n)}\right)+\varepsilon_n,
\qquad
\varepsilon_n\stackrel{\mathrm{iid}}{\sim}\mathcal N(0,\sigma^2).
\label{eq:observation-model}
\end{equation}
Here $\varepsilon_n$ is the scalar error in the $n$th observation, whereas
$\beps=(\varepsilon_1,\ldots,\varepsilon_N)^T$ denotes the complete
observation-noise vector.  Perturbation levels introduced below, such as
$\varepsilon_H(\btheta)$, are scalar quantities.

Let $w_n\ge0$ denote deterministic sampling or quadrature weights normalized
so that $\sum_{n=1}^Nw_n=1$; for Monte Carlo sampling one may take
$w_n=1/N$.  To construct both lag-zero and lag-one
$R\times C$ Hankel matrices, set $K_H=R+C-1$.  For a fixed phase vector
$\btheta$, snapshot $\nu=0,\ldots,d$, and $k=0,\ldots,K_H$, define the
noise-free finite-data probe
\begin{equation*}
m_{k,N}^{[\nu]}
=
\sum_{n=1}^{N}
w_n
Y\left(\bxi^{(n)}\right)
\K_{\A}
\left(
\bt^{[\nu]}(k),\bxi^{(n)}
\right),
\end{equation*}
and the observable probe
\begin{equation*}
\widehat m_k^{[\nu]}
=
\sum_{n=1}^{N}
w_n
y^{(n)}
\K_{\A}
\left(
\bt^{[\nu]}(k),\bxi^{(n)}
\right).
\end{equation*}
The corresponding population probe is
$m_k^{[\nu]}=\G_{\A}Y(\bt^{[\nu]}(k))$.  Thus the three levels
$m_k^{[\nu]}$, $m_{k,N}^{[\nu]}$, and $\widehat m_k^{[\nu]}$ represent,
respectively, the exact population quantity, its noise-free finite-data
approximation, and the actually observed noisy probe.

Define the observation-induced probe error
\begin{equation*}
e_k^{[\nu]}
=
\widehat m_k^{[\nu]}-m_{k,N}^{[\nu]}
=
\sum_{n=1}^{N}
w_n\varepsilon_n
\K_{\A}
\left(
\bt^{[\nu]}(k),\bxi^{(n)}
\right).
\end{equation*}
Stack all pairs $(\nu,k)$ in a fixed order and denote the resulting vectors by
$\widehat{\bm m}_{\mathrm{stk}}(\btheta)$,
$\bm m_{N,\mathrm{stk}}(\btheta)$, and
$\bm m_{\mathrm{stk}}(\btheta)$.  Then
\begin{equation}
\widehat{\bm m}_{\mathrm{stk}}(\btheta)
=
\bm m_{N,\mathrm{stk}}(\btheta)
+
B(\btheta)\beps,
\label{eq:probe-linear-noise}
\end{equation}
where $B(\btheta)\in\C^{(d+1)(K_H+1)\times N}$.  The row indexed by
$(\nu,k)$ has entries
\begin{equation*}
B(\btheta)_{(\nu,k),n}
=
w_n
\K_{\A}
\left(
\bt^{[\nu]}(k),\bxi^{(n)}
\right).
\end{equation*}
Equation~\eqref{eq:probe-linear-noise} separates the two finite-data effects.
The deterministic difference
$\bm m_{N,\mathrm{stk}}-\bm m_{\mathrm{stk}}$ is the finite-sampling or
quadrature error, whereas $B(\btheta)\beps$ is the additional perturbation
caused by observation noise.

\begin{proposition}[Exact conditional probe covariance]
\label{prop:probe-covariance}
Under \eqref{eq:observation-model}, conditional on the input design,
\begin{equation}
\operatorname{Cov}
\left(
\widehat{\bm m}_{\mathrm{stk}}
-
\bm m_{N,\mathrm{stk}}
\right)
=
\sigma^2
B(\btheta)B(\btheta)^*.
\label{eq:probe-covariance}
\end{equation}
\end{proposition}

\begin{proof}
Equation~\eqref{eq:probe-linear-noise} and
$\E[\beps\beps^T]=\sigma^2I_N$ give the result directly.
\end{proof}

The covariance is therefore determined by the polynomial basis, radii, phase
design, input points, and weights rather than being an unspecified nuisance
parameter.

\begin{theorem}[Gaussian probe-noise radius]
\label{thm:gaussian-probe-bound}
Let $\beps\sim\mathcal N(0,\sigma^2I_N)$ and let $B\in\C^{N_{\mathrm{pr}}\times N}$ be fixed.
For every $0<\delta<1$, with probability at least $1-\delta$,
\begin{equation*}
\norm{B\beps}_2
\le
\sigma
\left(
\norm{B}_F
+
\norm{B}_2
\sqrt{2\log\frac{1}{\delta}}
\right).
\end{equation*}
\end{theorem}

\begin{proof}
Write $\beps=\sigma\bm g$ with $\bm g\sim\mathcal N(0,I_N)$.  The map
$\bm g\mapsto\norm{B\bm g}_2$ is $\norm{B}_2$-Lipschitz.  Gaussian
concentration gives
\begin{equation*}
\Pp
\left(
\norm{B\bm g}_2
\ge
\E\norm{B\bm g}_2+t
\right)
\le
\exp
\left(
-\frac{t^2}{2\norm{B}_2^2}
\right).
\end{equation*}
Moreover,
$\E\norm{B\bm g}_2\le(\E\norm{B\bm g}_2^2)^{1/2}=\norm{B}_F$.
Choosing $t=\norm{B}_2\sqrt{2\log(1/\delta)}$ proves the result.
\end{proof}

Let $\Hh_{R,C}(u)$ denote the $R\times C$ Hankel matrix generated by a finite
sequence $u$.

\begin{lemma}[Hankel lifting bound]
\label{lem:hankel-lift}
Let $h=\min\{R,C\}$.  Then
\begin{equation*}
\norm{\Hh_{R,C}(u)}_2
\le
\norm{\Hh_{R,C}(u)}_F
\le
\sqrt{h}\norm{u}_2.
\end{equation*}
The same estimate holds for vertically stacked Hankel matrices when $u$ is
replaced by the stacked vector of their generating sequences.
\end{lemma}

\begin{proof}
If $q_k$ denotes the number of entries on the $k$-th Hankel anti-diagonal,
then $q_k\le h$ and
\begin{equation*}
\norm{\Hh_{R,C}(u)}_F^2
=
\sum_k q_k|u_k|^2
\le
h\norm{u}_2^2.
\end{equation*}
The spectral-norm bound is standard.  Summing the same estimate over stacked
blocks proves the final statement.
\end{proof}

For each $\nu=0,\ldots,d$ and $\kappa=0,1$, define the noise-free finite-data
and noisy Hankel matrices
\begin{equation*}
\mathcal H_{\kappa,N}^{[\nu]}
=
\left(
m_{u+v+\kappa,N}^{[\nu]}
\right)_{u,v},
\qquad
\widehat{\mathcal H}_\kappa^{[\nu]}
=
\left(
\widehat m_{u+v+\kappa}^{[\nu]}
\right)_{u,v},
\end{equation*}
where $0\le u\le R-1$ and $0\le v\le C-1$.  Their observation-induced
difference is the explicit Hankel perturbation
\begin{equation*}
E_{\mathrm{obs},\kappa}^{[\nu]}
=
\widehat{\mathcal H}_\kappa^{[\nu]}
-
\mathcal H_{\kappa,N}^{[\nu]}
=
\left(
e_{u+v+\kappa}^{[\nu]}
\right)_{u,v}.
\end{equation*}
Stacking the snapshots gives
\begin{equation*}
\mathcal J_{\kappa,N}
=
\begin{pmatrix}
\mathcal H_{\kappa,N}^{[0]}\\
\vdots\\
\mathcal H_{\kappa,N}^{[d]}
\end{pmatrix},
\qquad
\widehat{\mathcal J}_\kappa
=
\begin{pmatrix}
\widehat{\mathcal H}_\kappa^{[0]}\\
\vdots\\
\widehat{\mathcal H}_\kappa^{[d]}
\end{pmatrix},
\end{equation*}
and
\begin{equation*}
E_{\mathrm{obs},J,\kappa}
=
\widehat{\mathcal J}_\kappa-\mathcal J_{\kappa,N}
=
\begin{pmatrix}
E_{\mathrm{obs},\kappa}^{[0]}\\
\vdots\\
E_{\mathrm{obs},\kappa}^{[d]}
\end{pmatrix}.
\end{equation*}

\begin{corollary}[Observation-to-Hankel perturbation bound]
\label{cor:obs-hankel}
For a fixed phase $\btheta$, with probability at least $1-\delta$,
\begin{equation*}
\max_{\kappa=0,1}
\norm{E_{\mathrm{obs},J,\kappa}}_2
\le
\Delta_H(\btheta,\delta),
\end{equation*}
where
\begin{equation}
\Delta_H(\btheta,\delta)
=
\sqrt{h}\,\sigma
\left(
\norm{B(\btheta)}_F
+
\norm{B(\btheta)}_2
\sqrt{
2\log\frac{1}{\delta}
}
\right).
\label{eq:hankel-noise-radius}
\end{equation}
\end{corollary}

\begin{proof}
Apply Theorem~\ref{thm:gaussian-probe-bound} to the complete stacked probe
error $B(\btheta)\beps$.  For either lag, the generating error sequences form
a subvector of this stack, so Lemma~\ref{lem:hankel-lift} gives the stated
bound.
\end{proof}

The finite-data approximation contributes a second, deterministic perturbation.
Define
\begin{equation*}
E_{\mathrm{int},J,\kappa}(\btheta)
=
\mathcal J_{\kappa,N}(\btheta)-\mathcal J_\kappa(\btheta),
\qquad
\eta_H(\btheta)
=
\max_{\kappa=0,1}
\norm{E_{\mathrm{int},J,\kappa}(\btheta)}_2.
\end{equation*}
Then
\begin{equation*}
\widehat{\mathcal J}_\kappa-\mathcal J_\kappa
=
E_{\mathrm{int},J,\kappa}
+
E_{\mathrm{obs},J,\kappa},
\end{equation*}
so, for a fixed phase, the total perturbation is bounded with probability at
least $1-\delta$ by
$\eta_H(\btheta)+\Delta_H(\btheta,\delta)$.  If the sampling or quadrature rule
is exact for the required probe functionals, then $\eta_H(\btheta)=0$.

For the multi-phase analysis, set
\begin{equation*}
\beta_{\mathrm{op}}
=
\sup_{\btheta\in[0,2\pi)^d}
\norm{B(\btheta)}_2,
\qquad
\overline{\eta}_H
=
\sup_{\btheta\in[0,2\pi)^d}
\eta_H(\btheta).
\end{equation*}
For a fixed finite candidate set, fixed radii, and fixed finite input design,
both quantities are finite.  Define the phase-uniform total radius
\begin{equation}
\overline{\Delta}_{\mathrm{tot}}(\delta)
=
\overline{\eta}_H
+
\sqrt{h}\,\sigma\beta_{\mathrm{op}}
\left(
\sqrt{N}
+
\sqrt{
2\log\frac{1}{\delta}
}
\right).
\label{eq:uniform-hankel-radius}
\end{equation}

\begin{proposition}[Phase-uniform total Hankel perturbation]
\label{prop:uniform-hankel}
Under \eqref{eq:observation-model}, for every $0<\delta<1$, with probability at
least $1-\delta$,
\begin{equation*}
\sup_{\btheta\in[0,2\pi)^d}
\max_{\kappa=0,1}
\norm{
\widehat{\mathcal J}_\kappa(\btheta)
-
\mathcal J_\kappa(\btheta)
}_2
\le
\overline{\Delta}_{\mathrm{tot}}(\delta).
\end{equation*}
\end{proposition}

\begin{proof}
Write $\beps=\sigma\bm g$ with $\bm g\sim\mathcal N(0,I_N)$.  With probability
at least $1-\delta$,
$\norm{\bm g}_2\le\sqrt N+\sqrt{2\log(1/\delta)}$.  On this event,
simultaneously for every phase,
\begin{equation*}
\norm{B(\btheta)\beps}_2
\le
\sigma\beta_{\mathrm{op}}
\left(
\sqrt N+\sqrt{2\log\frac1\delta}
\right).
\end{equation*}
The Hankel lifting bound controls both joint lags, and the deterministic
integration contribution is bounded by $\overline{\eta}_H$.
\end{proof}

The fixed-phase radius \eqref{eq:hankel-noise-radius} is typically sharper
because it uses the actual probe matrix and its Frobenius norm.  The uniform
radius \eqref{eq:uniform-hankel-radius} is more conservative, but it controls
the same realized observation noise simultaneously over all subsequently
drawn phases.  The term $\overline{\eta}_H$ makes explicit that an end-to-end
statement relative to the population Hankel matrices must also account for
finite-sample integration error.

\begin{theorem}[Perturbation-calibrated singular-gap condition]
\label{thm:rank-gap}
Let $\mathcal J_0$ be an exact joint Hankel matrix of rank $s$ and let
$\widehat{\mathcal J}_0=\mathcal J_0+E$.  If
$\norm{E}_2\le\Delta$ and
$\sigma_s(\mathcal J_0)>2\Delta$, then
\begin{equation*}
\sigma_s(\widehat{\mathcal J}_0)>\Delta,
\qquad
\sigma_{s+1}(\widehat{\mathcal J}_0)\le\Delta.
\end{equation*}
Hence thresholding the singular values at $\Delta$ recovers the exact model
order $s$.
\end{theorem}

\begin{proof}
Weyl's inequality gives
$\sigma_s(\widehat{\mathcal J}_0)
\ge\sigma_s(\mathcal J_0)-\norm{E}_2>\Delta$.
Since $\sigma_{s+1}(\mathcal J_0)=0$, it also gives
$\sigma_{s+1}(\widehat{\mathcal J}_0)\le\norm{E}_2\le\Delta$.
\end{proof}

The radii also enter the perturbation mechanism.  By orthogonality,
\begin{equation*}
\E\left|\K_{\A}(\bt,\bXi)\right|^2
=
\sum_{\balpha\in\A}|\bt^{\balpha}|^2,
\end{equation*}
and along the phase orbit this becomes
$\sum_{\balpha\in\A}\br^{2\balpha}$.
Smaller radii can reduce probe-noise amplification but simultaneously
attenuate an active signal by $\br^{\balpha}$.  The radii are therefore
statistical design parameters rather than purely numerical regularizers.

The role of this subsection is to reduce the different finite-data error
sources to perturbation levels at the joint-Hankel level.  The spectral and
discrete decoding analysis below does not need to distinguish their origin once
$\widehat{\mathcal J}_\kappa-\mathcal J_\kappa$ has been controlled.

\subsection{Random phase redundancy and separation}

The exact recovery theory requires distinct Prony nodes, but numerical
stability depends more strongly on how well these nodes are separated.
This subsection quantifies two complementary aspects of random phase
encoding.  First, we estimate the probability that a single random phase
places two active modes too close on the unit circle.  Second, we show that
independent phases produce increasingly separated aggregate codewords on the
finite candidate set.  The first result motivates repeated phase recovery,
while the second provides a geometric interpretation and a design diagnostic
for the multi-phase ensemble.

A single phase vector maps a candidate multi-index to
$z_{\btheta}(\balpha)=\exp(\ii\btheta\cdot\balpha)$.  Although exact collisions
occur with probability zero, near-collisions can make a Vandermonde factor
poorly conditioned.  We therefore draw independent phase vectors
$\btheta^{(1)},\ldots,\btheta^{(Q)}$ and set
$\Theta=(\btheta^{(1)},\ldots,\btheta^{(Q)})$.  Every candidate multi-index is
assigned the codeword
\begin{equation*}
\mathcal C_{\Theta}(\balpha)
=
\left(
\exp\left(\ii\btheta^{(1)}\cdot\balpha\right),
\ldots,
\exp\left(\ii\btheta^{(Q)}\cdot\balpha\right)
\right).
\end{equation*}
Define the normalized code distance
\begin{equation}
\Delta_{\mathrm{code}}(\Theta)
=
\min_{\balpha\neq\bbeta\in\A}
\left[
\frac{1}{Q}
\sum_{q=1}^{Q}
\left|
\exp\left(\ii\btheta^{(q)}\cdot\balpha\right)
-
\exp\left(\ii\btheta^{(q)}\cdot\bbeta\right)
\right|^2
\right]^{1/2}.
\label{eq:code-distance}
\end{equation}

For an active support $\Sset$, define the single-phase separation
\begin{equation*}
\Delta_{\btheta}(\Sset)
=
\min_{\substack{\balpha,\bbeta\in\Sset\\\balpha\neq\bbeta}}
\left|
\exp\left(\ii\btheta\cdot\balpha\right)
-
\exp\left(\ii\btheta\cdot\bbeta\right)
\right|.
\end{equation*}

\begin{theorem}[Near-collision probability for one random phase]
\label{thm:random-near-collision}
Let $|\Sset|=s$ and let $\btheta$ be uniformly distributed on
$[0,2\pi)^d$.  For every $0<\eta\le2$,
\begin{equation*}
\Pp
\left(
\Delta_{\btheta}(\Sset)<\eta
\right)
\le
\binom{s}{2}
\frac{2}{\pi}
\arcsin
\left(
\frac{\eta}{2}
\right).
\end{equation*}
\end{theorem}

\begin{proof}
For fixed $\balpha\neq\bbeta$, let
$\bm{\nu}=\balpha-\bbeta\in\mathbb Z^d\setminus\{0\}$.  Since at least one
component of $\bm{\nu}$ is a nonzero integer, Haar invariance on the circle
implies that $\btheta\cdot\bm{\nu}\bmod 2\pi$ is uniform on $[0,2\pi)$.  If
$\phi$ is uniform on this interval, then
$|e^{\ii\phi}-1|=2|\sin(\phi/2)|$, and hence
\begin{equation*}
\Pp
\left(
\left|
e^{\ii\phi}-1
\right|
<
\eta
\right)
=
\frac{2}{\pi}
\arcsin
\left(
\frac{\eta}{2}
\right).
\end{equation*}
A union bound over the $\binom{s}{2}$ pairs proves the claim.
\end{proof}

\begin{proposition}[Aggregate separation of a random phase codebook]
\label{prop:codebook-separation}
Let $P=|\A|$, and let
$\btheta^{(1)},\ldots,\btheta^{(Q)}$ be independent and uniform on
$[0,2\pi)^d$.  For every $0<t<2$,
\begin{equation}
\Pp
\left(
\Delta_{\mathrm{code}}(\Theta)^2
\le
2-t
\right)
\le
\binom{P}{2}
\exp
\left(
-\frac{Qt^2}{8}
\right).
\label{eq:codebook-concentration}
\end{equation}
Consequently, a sufficient condition for
$\Delta_{\mathrm{code}}(\Theta)>\sqrt{2-t}$ with probability at least
$1-\delta$ is
\begin{equation}
Q
\ge
\frac{8}{t^2}
\log
\left(
\frac{P(P-1)}{2\delta}
\right).
\label{eq:phase-count-logP}
\end{equation}
\end{proposition}

\begin{proof}
For a fixed pair $\balpha\neq\bbeta$, set
$X_q=|\exp(\ii\btheta^{(q)}\cdot\balpha)
-\exp(\ii\btheta^{(q)}\cdot\bbeta)|^2$.  The argument of
Theorem~\ref{thm:random-near-collision} gives $\E X_q=2$ and
$0\le X_q\le4$.  Hoeffding's inequality yields
\begin{equation*}
\Pp
\left(
\frac{1}{Q}
\sum_{q=1}^{Q}
X_q
\le
2-t
\right)
\le
\exp
\left(
-\frac{Qt^2}{8}
\right).
\end{equation*}
A union bound over all pairs in $\A$ proves
\eqref{eq:codebook-concentration}; solving for $Q$ gives
\eqref{eq:phase-count-logP}.
\end{proof}

Proposition~\ref{prop:codebook-separation} concerns the aggregate geometry of
the phase ensemble; it does not state that every individual phase is well
conditioned.  Its role is to quantify the redundancy created by repeated phase
encoding and to provide a phase-design diagnostic.  The support-voting result
below instead depends on the success probability of the individual phase
decoder.

\subsection{Stable discrete decoding and phase voting}

We now connect the finite-data perturbation with the phase-dependent spectral
conditioning.  The argument has three steps: the joint-Hankel perturbation
controls the continuous spectral decoder, the resulting node and amplitude
errors are compared with the discrete decoding margins, and independent phase
repetitions are finally aggregated by voting.

For one phase, define the total joint-Hankel perturbation level
\begin{equation*}
\varepsilon_H(\btheta)
=
\max_{\kappa=0,1}
\norm{
\widehat{\mathcal J}_\kappa(\btheta)
-
\mathcal J_\kappa(\btheta)
}_2.
\end{equation*}
The continuous part of the decoder consists of an SVD signal-subspace
estimate, the reduced joint pencil in
\eqref{eq:joint-compressed-pencil}, and Vandermonde least-squares amplitude
fits.  The following local statement records only the stability property
needed for the PCE-specific discrete decoder.

\begin{lemma}[Local stability of the joint spectral decoder]
\label{lem:local-joint-stability}
Fix a phase $\btheta$ for which the assumptions of
Theorem~\ref{thm:exact-single} hold.  Then there exist
$\varepsilon_{\mathrm{loc}}(\btheta)>0$ and finite constants
$K_z(\btheta)$ and $K_a(\btheta)$ such that, whenever
$\varepsilon_H(\btheta)<\varepsilon_{\mathrm{loc}}(\btheta)$, the rank-$s$
SVD-compressed joint decoder can be labeled so that
\begin{align*}
\max_j
\left|
\widehat z^{(j)}-z^{(j)}
\right|
&\le
K_z(\btheta)\varepsilon_H(\btheta),
\\
\max_{j,\nu}
\left|
\widehat a_\nu^{(j)}-a_\nu^{(j)}
\right|
&\le
K_a(\btheta)\varepsilon_H(\btheta).
\end{align*}
The local radius may be chosen so that
\begin{equation*}
\varepsilon_{\mathrm{loc}}(\btheta)
\le
\frac14
\sigma_s
\left(
\mathcal J_0(\btheta)
\right).
\end{equation*}
\end{lemma}

\begin{proof}
Because $\sigma_s(\mathcal J_0)>0$, sufficiently small perturbations preserve
an $s$-dimensional singular subspace, with a perturbation proportional to
$\varepsilon_H$.  The reduced exact pencil has the simple spectrum
$\{z^{(j)}\}_{j=1}^s$ by Proposition~\ref{prop:joint-snapshot}; standard
subspace and matrix-pencil perturbation theory therefore gives a locally
Lipschitz labeling of its eigenvalues
\cite{GolubVanLoan2013,LiEtAl2022,LiuEtAl2024,DingEtAl2024}.  The two Hankel lags together contain every probe entry used in the amplitude
fits; finite-dimensional norm equivalence therefore bounds the corresponding
probe-sequence perturbations by a constant multiple of $\varepsilon_H$.
With distinct nodes, the Vandermonde least-squares maps for the $d+1$ snapshots
are locally Lipschitz in both the nodes and the probe data.  Shrinking the
neighborhood if necessary gives the stated constants and the final
singular-value condition.  A representative reduced-pencil estimate is included in the supplementary material.
\end{proof}

Let
\begin{equation*}
a_{\min}
=
\min_j|a_0^{(j)}|,
\qquad
d_{\mathrm{deg}}
=
\min_{1\le\ell\le d}d_\ell.
\end{equation*}
If the base and shifted amplitudes are perturbed by at most
$\varepsilon_a<a_{\min}/2$, then
\begin{equation*}
\left|
\widehat\chi_\ell^{(j)}
-
\chi_\ell^{(j)}
\right|
\le
\frac{4\varepsilon_a}{a_{\min}}.
\end{equation*}
Hence nearest-root decoding is exact whenever
$\varepsilon_a<a_{\min}d_{\mathrm{deg}}/8$.

Define the single-phase stability radius
\begin{equation}
\varepsilon_{\mathrm{stab}}(\btheta)
=
\min
\left\{
\varepsilon_{\mathrm{loc}}(\btheta),
\frac{
\Delta_{\btheta}(\Sset)
}{
3K_z(\btheta)
},
\frac{
a_{\min}d_{\mathrm{deg}}
}{
8K_a(\btheta)
}
\right\}.
\label{eq:single-phase-stability-radius}
\end{equation}

\begin{proposition}[Exact support recovery inside the stability radius]
\label{prop:single-phase-radius}
Under the exact-sparsity assumptions of Theorem~\ref{thm:exact-single}, if
$\varepsilon_H(\btheta)<\varepsilon_{\mathrm{stab}}(\btheta)$, then
thresholding the singular values of $\widehat{\mathcal J}_0$ at
$\varepsilon_H(\btheta)$ recovers the model order $s$, and the resulting
joint-snapshot decoder recovers the exact PCE support $\Sset$.
\end{proposition}

\begin{proof}
Since
$\varepsilon_{\mathrm{stab}}\le\varepsilon_{\mathrm{loc}}
\le\sigma_s(\mathcal J_0)/4$, Theorem~\ref{thm:rank-gap} gives the exact
model order.  Lemma~\ref{lem:local-joint-stability} and the second term in
\eqref{eq:single-phase-stability-radius} place each recovered node in a
disjoint neighborhood of its exact node, so the labeling is unique.  The
third term makes the amplitude error smaller than
$a_{\min}d_{\mathrm{deg}}/8$, and the root-of-unity margin then gives the exact
coordinate degrees.
\end{proof}

For phase $q$, let $\widetilde{\Sset}^{(q)}\subset\A$ be the support returned
by the joint decoder and define
\begin{equation}
v_Q(\balpha)=\frac{1}{Q}\sum_{q=1}^{Q}\bmone
\left\{\balpha\in\widetilde{\Sset}^{(q)}\right\}.
\label{eq:vote-score}
\end{equation}
For a threshold $0<\tau<1$, define
\begin{equation}
\widehat{\Sset}_{\tau}=\left\{\balpha\in\A:v_Q(\balpha)\ge\tau\right\},
\qquad
\widehat s=\left|\widehat{\Sset}_{\tau}\right|.
\label{eq:majority-estimator}
\end{equation}
This thresholded estimator is used when the model order is unknown.  In
controlled ablation studies where the true order $s$ is supplied to every
method, we also use the oracle-order estimator
\begin{equation*}
\widehat{\Sset}^{\mathrm{top}\text{-}s}=\operatorname{Top}_s
\left\{v_Q(\balpha): \balpha\in\A\right\},
\end{equation*}
which retains the $s$ largest vote scores, with a fixed deterministic tie-breaking rule.  When neither the order nor a
reliable perturbation threshold is available, each phase may deliberately
over-recover up to a prescribed ceiling $s_{\max}$ and use phase persistence
for practical order selection.

\begin{theorem}[Conditional phase-voting concentration]
\label{thm:phase-voting}
Fix the observed dataset $\Dset_N$, including its realized observation noise.
Draw the phases independently, and let the phase decoder be deterministic
conditional on $\Dset_N$ and the phase.  Define
$p_{\balpha}(\Dset_N)
=\Pp_{\btheta}(\balpha\in\widetilde{\Sset}(\btheta;\Dset_N)\mid\Dset_N)$.
Suppose $p_{\mathrm T}>\tau>p_{\mathrm F}$ and
\begin{equation*}
\min_{\balpha\in\Sset}p_{\balpha}(\Dset_N)\ge p_{\mathrm T},
\qquad
\max_{\bbeta\in\A\setminus\Sset}p_{\bbeta}(\Dset_N)\le p_{\mathrm F}.
\end{equation*}
Then
\begin{align}
\Pp\left(\widehat{\Sset}_{\tau}\neq\Sset\mid\Dset_N\right)\le
|\Sset|\exp\left(-2Q\left(p_{\mathrm T}-\tau\right)^2\right)+\left(|\A|-|\Sset|\right)\exp\left(-2Q\left(\tau-p_{\mathrm F}\right)^2\right).
\label{eq:voting-bound}
\end{align}
\end{theorem}

\begin{proof}
For each fixed candidate, the vote indicators are independent Bernoulli
variables conditional on $\Dset_N$ because the phases are independent.
Hoeffding's inequality gives the corresponding lower-tail bound for true
modes and upper-tail bound for false modes.  A union bound over the candidate
set proves \eqref{eq:voting-bound}.
\end{proof}

To connect the phase-dependent stability radius with the finite-data bound,
fix $0<\delta<1$ and define the bad-phase probability
\begin{equation}
\pi_\delta=\Pp_{\btheta}\left(\varepsilon_{\mathrm{stab}}(\btheta)\le\overline{\Delta}_{\mathrm{tot}}(\delta)\right).
\label{eq:bad-phase-probability}
\end{equation}

\begin{theorem}[End-to-end majority-voting recovery]
\label{thm:end-to-end}
Assume exact sparsity with support $\Sset\subset\A$, let the phase vectors be
independent and uniform on $[0,2\pi)^d$, and suppose the observation noise
satisfies \eqref{eq:observation-model}.  If $\pi_\delta<1/2$, then majority
voting at threshold $1/2$ satisfies
\begin{equation}
\Pp\left(\widehat{\Sset}_{1/2}\neq\Sset\right)
\le\delta+|\A|\exp\left[-2Q\left(\frac12-\pi_\delta\right)^2\right].
\label{eq:end-to-end-bound}
\end{equation}
\end{theorem}

\begin{proof}
By Proposition~\ref{prop:uniform-hankel}, with probability at least
$1-\delta$ the total joint-Hankel perturbation is bounded by
$\overline{\Delta}_{\mathrm{tot}}(\delta)$ simultaneously for every phase
vector.  Condition on this event and on the realized dataset.
Proposition~\ref{prop:single-phase-radius} shows that every phase satisfying
\begin{equation*}
\varepsilon_{\mathrm{stab}}(\btheta)>\overline{\Delta}_{\mathrm{tot}}(\delta)
\end{equation*}
returns the exact support.  Consequently each true mode is selected with
conditional probability at least $1-\pi_\delta$, whereas a false mode can be
selected only on a phase belonging to the complementary bad-phase event and
therefore has conditional selection probability at most $\pi_\delta$.
Hoeffding's inequality and a union bound over $\A$ give the exponential term in
\eqref{eq:end-to-end-bound}.  Adding the probability $\delta$ of leaving the
phase-uniform perturbation event completes the proof.
\end{proof}

The theorem separates four mechanisms: finite-data integration error enters
through $\overline{\eta}_H$, ordinary observation noise through the remaining
part of $\overline{\Delta}_{\mathrm{tot}}(\delta)$, phase-conditioned spectral
stability through $\varepsilon_{\mathrm{stab}}(\btheta)$, and random redundancy
through the exponential factor in $Q$.  When the quadrature rule is exact for
the required probe functionals, $\overline{\eta}_H=0$.

\subsection{Coefficient refit and practical CH-PC algorithm}

The spectral stage is used to identify support.  After a support
$\widehat{\Sset}$ has been selected, the final coefficients are estimated
directly from the original observations by the restricted weighted
least-squares problem
\begin{equation}
\widehat{\bm c}_{\widehat{\Sset}}=\operatorname*{argmin}_{\bm c}\sum_{n=1}^{N}w_n\left|
y^{(n)}-\sum_{\balpha\in\widehat{\Sset}}c_{\balpha}\Psi_{\balpha}
\left(\bxi^{(n)}\right)\right|^2.
\label{eq:restricted-refit}
\end{equation}
This separates support identification from coefficient estimation and avoids
using noisy Vandermonde amplitudes as final coefficient estimates.

A practical CH-PC implementation is summarized as follows.

\begin{enumerate}[label=\textbf{Step \arabic*.},leftmargin=1.8cm]
\item
Choose the finite candidate set $\A$, radii $\br$, Hankel dimensions $R, C$,
and a model-order ceiling $s_{\max}\le\min\{R,C\}$.

\item
Draw $Q$ independent phase vectors.  For moderate candidate sets, one may
optionally compare several phase ensembles using the code distance
\eqref{eq:code-distance}; exhaustive pairwise screening is not required by the
method and should be avoided when $|\A|$ is very large.

\item
Using the same observed model outputs, compute the base and coordinate-shifted
probe sequences for every phase.  Increasing $Q$ changes only postprocessing
and does not require additional evaluations of the forward model.

\item
For each phase, form the noisy joint Hankel pair, estimate the common signal
subspace, and solve the SVD-compressed common-node pencil.  A calibrated total
perturbation bound may be used for rank selection when the integration error
is controlled; otherwise over-recover up to $s_{\max}$ and rely on
cross-phase persistence.

\item
Fit the base and shifted amplitudes using the recovered common nodes and decode
candidate multi-indices from the root-of-unity amplitude ratios.

\item
Compute the phase vote scores \eqref{eq:vote-score} and retain the
majority-supported modes using \eqref{eq:majority-estimator}.  Their
cardinality is the estimated effective model order.

\item
Refit the PCE coefficients on the selected support by
\eqref{eq:restricted-refit}.
\end{enumerate}

\section{Approximate sparsity, sensitivity, and methodological scope}
\label{sec:approx-sparse}

The exact recovery theory identifies a finite active support.  In practical
PCE approximations, however, weak nonzero coefficients are often present.
Rather than introducing a second recovery theory, we show that these weak terms
enter the coded Hankel construction as an additional structured perturbation.
This gives a resolution-based interpretation of dominant support and connects
naturally with coefficient-energy sensitivity measures.  We then record the
main modeling and design issues that delimit the present scope.

\subsection{Approximate sparsity as a structured perturbation}

Suppose
\begin{equation*}
Y_{\A}=\sum_{\balpha\in\Sset}c_{\balpha}\Psi_{\balpha}+\sum_{\balpha\in\Tset}c_{\balpha}\Psi_{\balpha},
\qquad
\Tset=\A\setminus\Sset,
\end{equation*}
where the second sum is weak but not zero.  For the base phase probe,
$m_k=m_k^{\Sset}+m_k^{\Tset}$ and
\begin{equation}
\left|m_k^{\Tset}\right|\le\eta_{\br},
\qquad
\eta_{\br}=\sum_{\balpha\in\Tset}|c_{\balpha}|\br^{\balpha}.
\label{eq:tail-probe-radius}
\end{equation}
The same bound holds for every coordinate-shifted snapshot because the
additional root-of-unity factors have unit modulus.

Let $E_{\mathrm{tail},J,\kappa}$ denote the joint Hankel matrix obtained by
lifting only these tail probes.  Since each of its $(d+1)RC$ entries has
magnitude at most $\eta_{\br}$,
\begin{equation}
\norm{E_{\mathrm{tail}, J, \kappa}}_2
\le\norm{E_{\mathrm{tail},J,\kappa}}_F
\le\sqrt{(d+1)RC}\, \eta_{\br},
\qquad
\kappa=0, 1.
\label{eq:tail-hankel-radius}
\end{equation}
Thus approximate sparsity enters the dominant-support problem as a third
structured perturbation, in addition to finite-data integration error and
observation noise.

This bound also clarifies the interpretation of the recovered support.
Voting does not prove that omitted PCE coefficients are mathematically zero.
Rather, CH-PC identifies modes that are resolvable relative to the combined
scales of dominant amplitudes, radius attenuation, weak-tail energy,
finite-data integration error, observation noise, and spectral conditioning.
This interpretation is consistent with the coefficient-energy definition of a
dominant PCE representation.

\subsection{Sensitivity information}

For an orthonormal PCE, the nonconstant coefficient energy equals the response
variance, and the total Sobol index of input $\Xi_\ell$ is
\begin{align*}
V
&=
\Var(Y_{\A})
=
\sum_{\substack{\balpha\in\A\\\balpha\neq0}}|c_{\balpha}|^2,
\\
S_\ell^T
&=\displaystyle
\sum_{\substack{\balpha\in\A\\\alpha_\ell>0}}|c_{\balpha}|^2
/V.
\end{align*}
Thus a recovered multi-index records both the variables participating in a
mode and the corresponding contribution to the variance decomposition
\cite{Sudret2008}.

\begin{proposition}[Sobol stability under coefficient error]
\label{prop:sobol-stability}
Let $\bm c$ and $\widehat{\bm c}$ be the exact and estimated vectors of
nonconstant coefficients on the same finite candidate set and suppose
$\norm{\widehat{\bm c}-\bm c}_2\le\varepsilon_c$.  Set
$\Delta_c=\varepsilon_c(2\norm{\bm c}_2+\varepsilon_c)$.  If
$V=\norm{\bm c}_2^2>0$ and
$\widehat V=\norm{\widehat{\bm c}}_2^2>0$, then
\begin{equation*}
\left|\widehat S_\ell^T-S_\ell^T\right|
\le {2\Delta_c}/{V}.
\end{equation*}
\end{proposition}

\begin{proof}
Let $V_\ell^T$ and $\widehat V_\ell^T$ denote the exact and estimated total-effect
numerators.  The difference of the squared norms on any subset of coefficient indices is
bounded by $\Delta_c$.  Therefore
\begin{align*}
\left|\widehat V_\ell^T/\widehat V-V_\ell^T/V\right|
\le |V_\ell^T-\widehat V_\ell^T|/V+\widehat V_\ell^T\left|1/\widehat V-1/V\right|
\le\Delta_c/V+|V-\widehat V|/V
\le2\Delta_c/V.
\end{align*}
\end{proof}

\subsection{Present scope and unresolved issues}

The probe covariance \eqref{eq:probe-covariance} suggests a noise-aware design
problem: phase separation should be balanced against the amplification encoded
by the probe matrices and against attenuation of high-order modes through
$\br^{\balpha}$.  A diagnostic such as
$\Delta_{\mathrm{code}}(\Theta)/
\max_q\norm{B(\btheta^{(q)})}_2$
captures only part of this tradeoff; optimal joint phase--radius design
remains open.

The Gaussian model in \eqref{eq:observation-model} isolates ordinary
observation error.  The linear probe representation extends to independent
sub-Gaussian noise through the corresponding concentration inequalities, and
heteroscedastic noise can be handled by replacing $\sigma^2I$ with a known or
estimated observation covariance.  Heavy-tailed contamination and outliers
would require robust estimators of the probe functionals.

Finite-sample integration error is isolated in the analysis through
$\eta_H(\btheta)$ and its phase-uniform counterpart
$\overline{\eta}_H$, but the present paper does not derive sample-complexity
bounds for these quantities.  The generating identity is exact at the
population level, whereas the data-based probes approximate the required
expectations using the available design.  Sparse grids, quasi-Monte Carlo
rules, importance sampling, and optimized experimental designs may therefore
be as important as the spectral solver in high stochastic dimension.  In the
experiments below, quadrature-based examples are used first to make this
integration error negligible and isolate observation-noise robustness before
a PDE example is introduced.

Finally, CH-PC should be viewed as a spectral support-identification framework,
not as a universal replacement for sparse regression.  Regression methods and
CH-PC have different conditioning mechanisms: sampled-dictionary geometry for
the former, and probe noise, Hankel singular gaps, and phase-node separation
for the latter.  The numerical study below therefore emphasizes recovery
behavior and the roles of joint snapshots and phase redundancy, without
presuming a fixed ordering of methods.


\section{Numerical Experiments}
\label{sec:numerics}

The numerical study is organized to test four distinct aspects of CH-PC:
exact recovery, robustness to observation noise, unknown-order recovery by
phase persistence, and recovery of dominant modes for a stochastic PDE.
Examples~1--3 use the same synthetic Legendre benchmark so that changes in
performance can be attributed to the recovery mechanism rather than to a
change of model.

\paragraph{Common synthetic benchmark.}
For Examples~1--3, $d=3$ and the candidate set is
$$
\A=\{\balpha\in\N_0^3:0\leq\alpha_\ell\leq4,\ \ell=1,2,3\},
\qquad |\A|=125.
$$
The true support and coefficients are
$$
\Sset=
\{(1,0,0),(0,2,0),(1,1,0),(0,0,3)\},
\qquad
\bm c_{\Sset}=(1.2,-0.8,0.5,0.35)^T.
$$
The model outputs are evaluated by the tensor Gauss--Legendre rule with five
nodes per coordinate, hence $N=125$.  When noise is added, we use
$$
y_n^\delta=y_n+\sigma_\delta\varepsilon_n,
\qquad
\varepsilon_n\sim\mathcal N(0,1),
\qquad
\sigma_\delta=
\delta\left(\sum_{n=1}^N w_n y_n^2\right)^{1/2},
$$
and refer to $\delta$ as the relative RMS observation-noise level.
Unless stated otherwise, the radii are $r_\ell=0.8$.  LARS and OMP receive
the same observations, weights, candidate set, and prescribed model order as
CH-PC; coefficient estimates are recomputed by the common restricted
least-squares refit \eqref{eq:restricted-refit}.

\subsection{Example 1: exact sparse recovery}
\label{subsec:exact-numerics}

We first verify the exact finite-rank mechanism using one phase realization
and no observation noise.  The joint Hankel matrices use $R=C=6$ and
$K=12$ probe values, with the true order $s=4$ supplied to the decoder.
The recovered support agrees with $\Sset$ up to permutation, and the relative
coefficient error is
$$
\frac{\norm{\widehat{\bm c}_{\Sset}-\bm c_{\Sset}}_2}
{\norm{\bm c_{\Sset}}_2}
=
2.35\times10^{-15}.
$$
For this phase, the minimum pairwise phase-code distances over the full
candidate set and over the active support are $1.39\times10^{-2}$ and
$6.23\times10^{-1}$, respectively.

Figure~\ref{fig:singular_values} shows a sharp rank-four structure: the first
four singular values lie between approximately $10^1$ and $10^{-1}$, while
the remaining singular values are below $10^{-13}$.  This numerically
confirms the exact Hankel-rank characterization in
Theorem~\ref{thm:exact-single}.

\begin{figure}[t]
\centering
\includegraphics[width=0.70\linewidth]{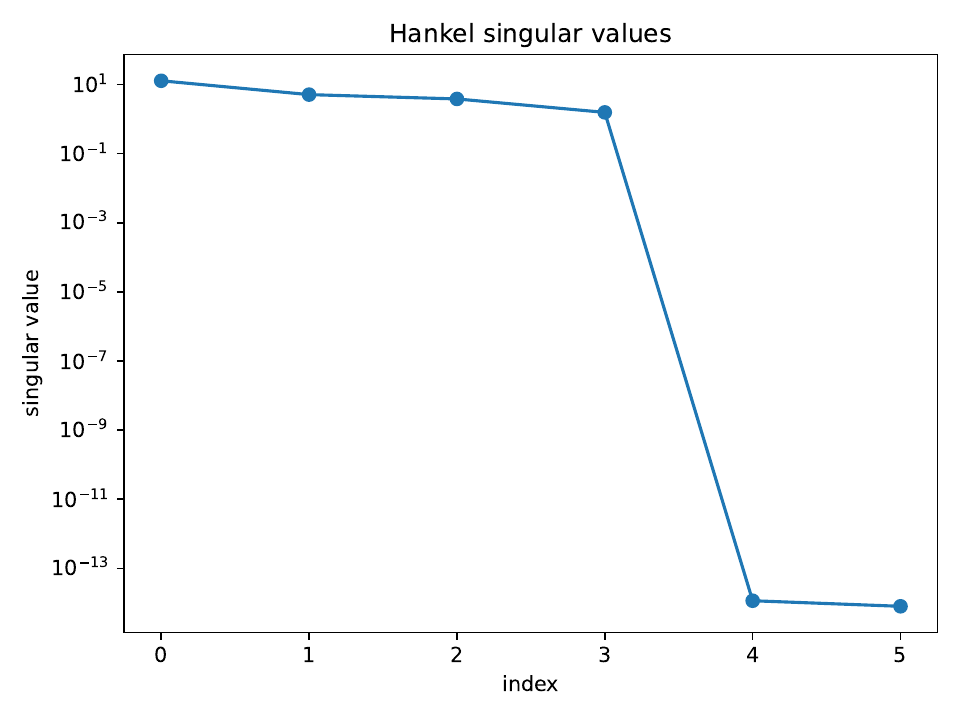}
\caption{Singular values of the joint Hankel matrix in Example~1.  The sharp
gap after the fourth singular value agrees with the true model order $s=4$.}
\label{fig:singular_values}
\end{figure}

\subsection{Example 2: stabilization under observation noise}
\label{subsec:noise-stability}

We next keep the benchmark fixed and compare three spectral variants:
a single-phase decoder based on the base Hankel pair, a one-phase
joint-snapshot decoder, and the complete CH-PC procedure with $Q=9$ phase
encodings and voting.  Here $R=C=6$, $K=12$, and the true order $s=4$ is
supplied to all methods.  For each
$$
\delta\in\{0,0.01,0.05,0.10,0.20,0.30\},
$$
we use 200 independent noise realizations.

All methods recover the exact support in every trial up to $10\%$ noise.
At $20\%$ noise the single-phase frequency decreases to $0.970$, while the
joint-snapshot and full CH-PC procedures remain at $1.000$.  At $30\%$
noise the corresponding frequencies are
$$
0.580,\qquad 0.825,\qquad 0.985,
$$
respectively.  Thus joint stacking provides a substantial gain within one
phase, and independent phase voting provides an additional gain on top of
the joint-snapshot construction.  LARS and OMP retain exact support recovery
throughout this small orthogonal benchmark, so this experiment is used to
isolate the stabilization mechanisms of CH-PC rather than to claim a
universal advantage over sparse regression.

\begin{figure}[t]
\centering
\includegraphics[width=0.74\linewidth]{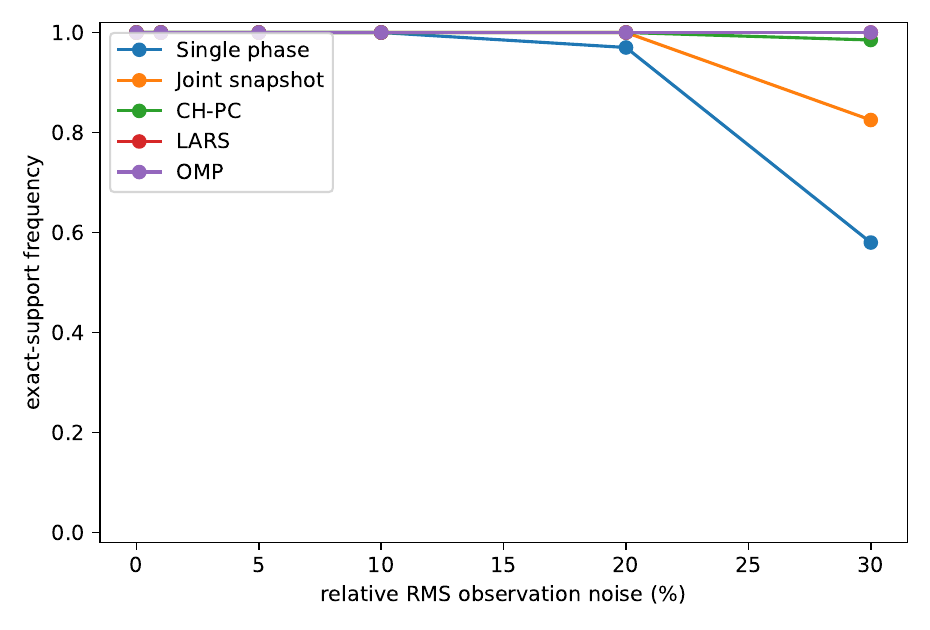}
\caption{Exact-support frequency versus relative RMS observation noise for
Example~2.  Joint snapshots improve the one-phase decoder, and multi-phase
voting provides further stabilization.}
\label{fig:noise-exact-support}
\end{figure}

After support selection, CH-PC, LARS, and OMP use the same restricted
weighted least-squares refit.  Their coefficient-error curves therefore
essentially coincide; the median relative errors for the correctly selected
four-mode model increase from $1.78\times10^{-3}$ at $1\%$ noise to
$5.11\times10^{-2}$ at $30\%$ noise.  We omit the nearly coincident
coefficient-error plot because it does not distinguish the support-selection
methods.

\subsection{Example 3: unknown order and phase persistence}
\label{subsec:unknown-order}

We now remove the known-order assumption while retaining the same benchmark.
The decoder is given only the ceiling $s_{\max}=6$ and uses $R=C=7$,
$K=14$.  Each phase decoder may therefore over-recover, and the final support
is obtained from the majority estimator \eqref{eq:majority-estimator}; its
cardinality is the estimated order $\widehat s$.

First fix $Q=13$ and vary the observation noise.  For $1\%$, $5\%$, and
$10\%$ noise, the exact support-and-order frequencies are $0.975$, $0.970$,
and $0.980$, respectively.  At $20\%$ noise the frequency is $0.950$, with
mean selected order $4.010$ and mean recall $0.9938$; at $30\%$ noise it
decreases to $0.735$, while the mean recall remains $0.9388$.
The noiseless over-ranked calculation retains one persistent extra mode and
therefore gives $\widehat s=5$.  This does not contradict
Theorem~\ref{thm:exact-single}: for exact data, direct singular-value rank
detection is the appropriate mechanism, whereas the present test
deliberately studies persistence filtering after over-recovery.

\begin{figure}[t]
\centering
\includegraphics[width=0.68\linewidth]{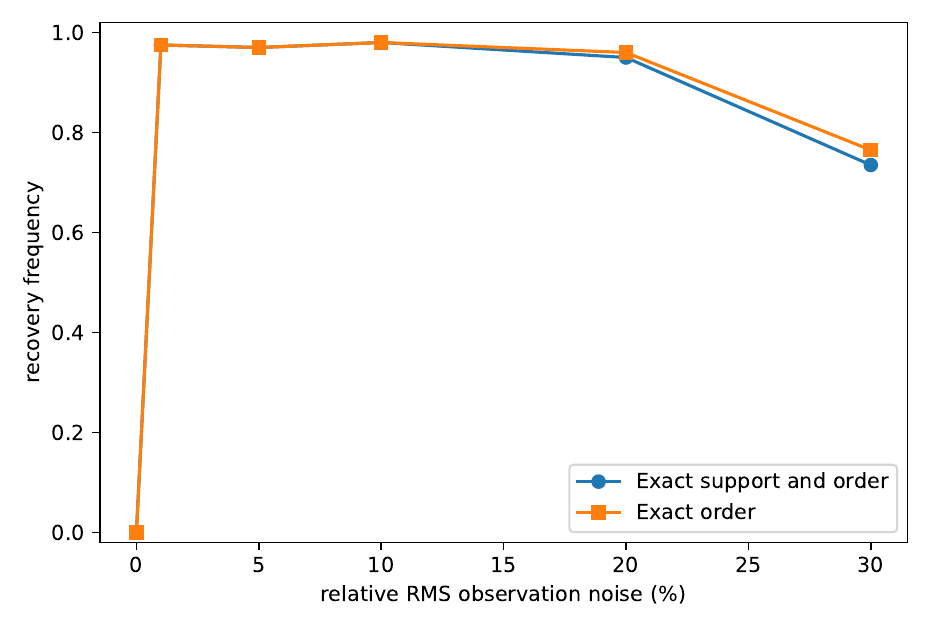}
\caption{Unknown-order recovery versus relative RMS observation noise in
Example~3.  The true order is not supplied to the algorithm.}
\label{fig:ex3-unknown-order-noise}
\end{figure}

To isolate the effect of phase redundancy, we next fix the noise level at
$30\%$ and vary
$$
Q\in\{3,5,7,9,13,17\}.
$$
For each $Q$, results are averaged over six independent phase ensembles and
60 noise realizations per ensemble.  The exact support-and-order frequency
increases from $0.444$ for $Q=3$ to $0.714$ for $Q=7$, $0.806$ for $Q=9$,
and $0.900$ for $Q=17$.  Over the same range, the mean selected order moves
from $3.839$ to $3.978$.  This trend is consistent with the aggregate
code-separation mechanism in Section~\ref{sec:robust}: repeated independent
encodings reduce the influence of poorly conditioned individual phases.

\begin{figure}[t]
\centering
\includegraphics[width=0.68\linewidth]{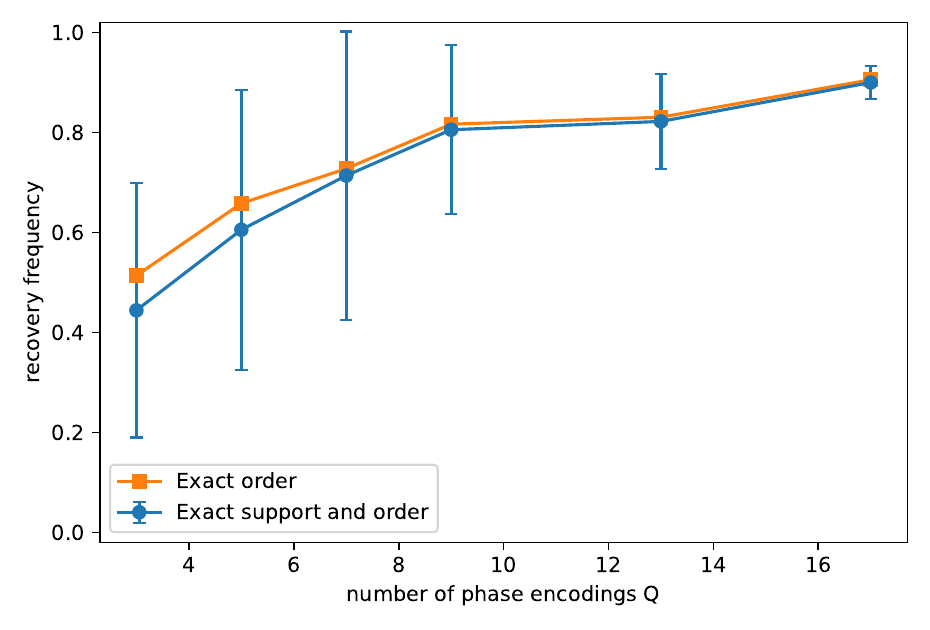}
\caption{Unknown-order recovery at $30\%$ noise as the number of phase
encodings increases.  Error bars show variation across independent phase
ensembles.}
\label{fig:ex3-phase-count}
\end{figure}

\subsection{Example 4: dominant modes for a stochastic Darcy problem}
\label{subsec:darcy}

The final example replaces the synthetic PCE by a quantity of interest
generated from the stochastic Darcy problem
\begin{align}
-\nabla\cdot\left(a(\mathbf x,\bxi)\nabla u(\mathbf x,\bxi)\right)
&=1,
\qquad \mathbf x\in D=(0,1)^2,
\\
u(\mathbf x,\bxi)&=0,
\qquad \mathbf x\in\partial D.
\end{align}
The four independent parameters satisfy $\xi_\ell\sim\mathcal U[-1,1]$, and
the log-permeability is
$$
\log a(\mathbf x,\bxi)
=
\sum_{\ell=1}^{4} b_\ell\xi_\ell\varphi_\ell(\mathbf x),
\qquad
(b_1,b_2,b_3,b_4)=(0.7,0.5,0.4,0.3),
$$
with
\begin{align*}
\varphi_1&=\sin(\pi x_1)\sin(\pi x_2),&
\varphi_2&=\cos(\pi x_1)\sin(\pi x_2),\\
\varphi_3&=\sin(\pi x_1)\cos(\pi x_2),&
\varphi_4&=\cos(2\pi x_1)\cos(\pi x_2).
\end{align*}
The PDE is discretized by finite differences with harmonic averaging of the
permeability.  The quantity of interest is the spatial mean of $u$ over
$D_R=\{\mathbf x\in D:x_1>1/2\}$.

The PCE candidate set contains the nonconstant total-degree Legendre modes
$$
\A=\{\balpha\in\N_0^4:1\leq|\balpha|_1\leq3\},
\qquad |\A|=34.
$$
Because the PDE response is not exactly sparse, a reference dominant support
is obtained from a higher-accuracy seven-point tensor Gauss--Legendre
projection.  The four largest nonconstant coefficients correspond to
$$
\Sset_{\rm ref}
=
\{(1,0,0,0),(0,1,0,0),(1,1,0,0),(2,0,0,0)\},
$$
with approximate values
$$
-4.786\times10^{-3},\quad
2.812\times10^{-3},\quad
-5.02\times10^{-4},\quad
3.52\times10^{-4}.
$$

Recovery uses a five-point rule in each stochastic coordinate, hence
$N=625$, with $r_\ell=0.75$, $R=C=9$, $K=18$, and $Q=31$.  The number of
dominant modes is fixed at four, while their multi-indices are unknown.
For
$$
\delta\in\{0,0.02,0.05,0.10,0.20\},
$$
100 independent noise realizations are used at each nonzero level.

CH-PC, LARS, and OMP recover the complete four-mode reference support in
every trial through $20\%$ relative RMS observation noise, as shown in
Figure~\ref{fig:darcy-support}.  Since the same support is selected, the
subsequent restricted least-squares coefficient estimates are essentially
identical; their median relative errors remain below $1.6\times10^{-2}$ at
$20\%$ noise.  The purpose of this example is therefore not to rank the
three support-selection methods, but to verify that the CH-PC construction
extends from exactly sparse synthetic PCEs to dominant-mode identification
for a nonlinear stochastic PDE response.

\begin{figure}[t]
\centering
\includegraphics[width=0.72\linewidth]{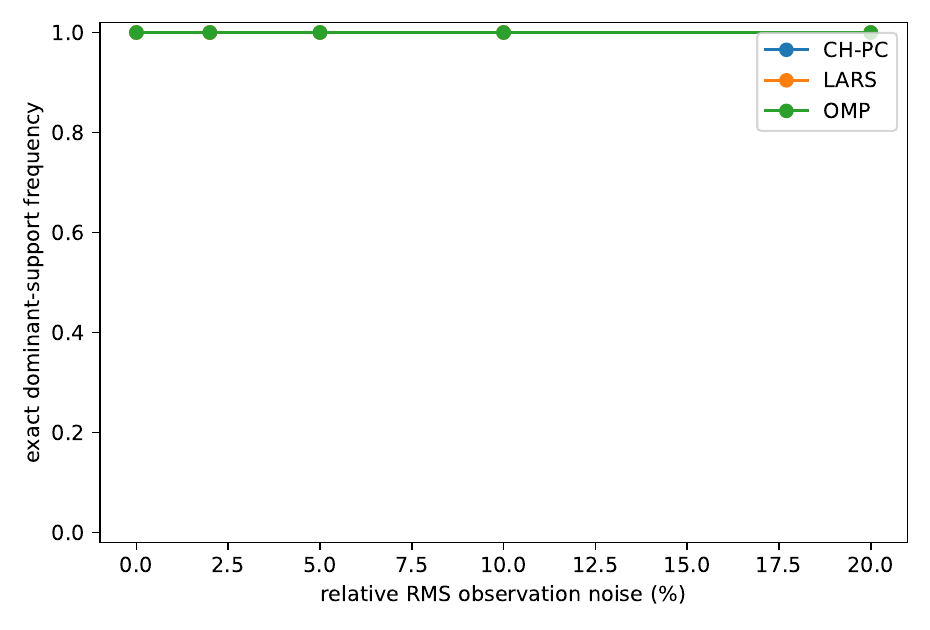}
\caption{Exact recovery frequency of the four reference dominant modes for
the stochastic Darcy problem.  All three methods recover the same dominant
support throughout the tested noise range.}
\label{fig:darcy-support}
\end{figure}

\section{Conclusion}
\label{sec:conclusion}

We developed coded Hankel polynomial chaos as a spectral framework for
identifying dominant PCE modes from input--output information.  The starting
point is a finite generating transform that converts orthogonal
polynomial-chaos coefficients into a coefficient-generating polynomial.  Geometric phase
sampling then produces a finite exponential sum, so model order and spectral
nodes are encoded by low-rank Hankel matrices.  Coordinate shifts preserve the
Prony nodes while attaching finite root-of-unity labels to the amplitudes,
turning continuous spectral recovery into coordinate-wise integer decoding.
This provides a direct route from observed responses to active stochastic
multi-indices---and therefore to the variables, polynomial orders, and
interactions represented by the response---rather than an iterative search
through candidate basis functions.

For finite observations, the exact spectral representation is retained but
perturbed.  Population, finite-data, and observed probes are kept distinct, so
integration error and observation error can be tracked separately at the
joint-Hankel level.  Joint coordinate snapshots exploit repeated observations
of the same spectral nodes, while independent random phase maps re-encode the
same discrete support.  The phase-separation results and majority-voting bound
show quantitatively how this redundancy suppresses dependence on a single
ill-conditioned encoding.  Over-recovery followed by persistence filtering
also provides a practical route to unknown model order.

A second contribution concerns computational structure.  For tensor-product
candidate sets, the generating kernel factorizes exactly into one-dimensional
polynomial sums, so CH-PC can form spectral probes without assembling the full
multivariate PCE design matrix.  This gives the method a complexity profile
different from dictionary-based sparse regression, although the total cost
also depends on the number of phase decoders and no universal runtime ordering
is claimed.

The numerical examples support this interpretation through exact and noisy
Legendre benchmarks, unknown-order recovery by phase persistence, and a
stochastic Darcy problem in which the dominant PCE modes are induced by the
PDE response rather than prescribed in advance.  The next theoretical
questions are sharper sample-complexity bounds in terms of observation count,
active amplitude, phase separation, and code redundancy; noise-aware
optimization of radii and phases; and structured probe evaluation for
non-tensor candidate sets.  These directions would place the spectral and
regression viewpoints on a common quantitative footing and clarify the regimes
in which each is most effective.

\appendix

\section{A square-pencil perturbation estimate}
\label{app:pencil}

This appendix records the standard reduced-pencil estimate used in the local
stability argument of Section~\ref{sec:robust}.  It is separated from the main
text because the estimate is supporting linear algebra rather than a
PCE-specific ingredient.  Let
$T=G_0^{-1}G_1$ and
$\widehat T=(G_0+E_0)^{-1}(G_1+E_1)$.

\begin{proposition}
If $\rho_0=\norm{G_0^{-1}E_0}_2<1$, then
\begin{equation}
\norm{\widehat T-T}_2
\le
\frac{\norm{G_0^{-1}}_2}{1-\rho_0}
\left(\norm{E_1}_2+\norm{E_0}_2\norm{T}_2\right).
\label{eq:pencil-perturbation}
\end{equation}
If $T=XZX^{-1}$ is diagonalizable, every eigenvalue of $\widehat T$ lies
within $\kappa_2(X)\norm{\widehat T-T}_2$ of the exact spectrum.
\end{proposition}

\begin{proof}
Use
$(G_0+E_0)^{-1}=(I+G_0^{-1}E_0)^{-1}G_0^{-1}$ and
$\norm{(I+G_0^{-1}E_0)^{-1}}_2\le(1-\rho_0)^{-1}$.  The first estimate follows
by subtraction and the second is the Bauer--Fike theorem.
\end{proof}

For coordinate $\ell$, the minimum distance between adjacent root-of-unity
codes is $d_\ell=2\sin(\pi/(p_\ell+1))$.  Hence nearest-root decoding is
exact whenever
$|\widehat\chi_\ell^{(j)}-\chi_\ell^{(j)}|<d_\ell/2$.

\section{Regression baselines}
\label{app:baselines}

The main text treats LARS and OMP only as established reference methods.  For
all numerical comparisons, they receive exactly the same model evaluations,
input points, weights, and finite PCE candidate set as CH-PC.  When the candidate
set is denoted by $\A$, the corresponding weighted regression matrix has entries
$\Phi_{n,\balpha}=\sqrt{w_n}\Psi_{\balpha}(\bxi^{(n)})$.  In oracle-order
experiments, the true number of retained modes is supplied to all three
methods so that the comparison concerns support identification rather than
model-order selection.  After support selection, all reported coefficients are
recomputed by the same restricted weighted least-squares refit
\eqref{eq:restricted-refit}.  Thus the numerical comparison isolates the
support-selection stage and does not depend on different coefficient-refitting
conventions.  The implementations use the LARS and orthogonal matching pursuit
routines in scikit-learn.

\bibliographystyle{plain}  
\bibliography{myref}

\end{document}